\documentclass[letterpaper, 10 pt, conference]{ieeeconf}  

\IEEEoverridecommandlockouts                              

\usepackage{amsmath}

\usepackage{amsthm}
\usepackage{amssymb}
\usepackage{float}
\usepackage{multirow}

\usepackage{ulem} 
\usepackage{xcolor} 
\usepackage{colortbl} 

\usepackage[ruled]{algorithm2e}
\usepackage{algorithmic}

\useunder{\uline}{\ul}{}

\usepackage{graphicx}

\makeatletter
\let\NAT@parse\undefined
\makeatother

\usepackage{hyperref}  
\hypersetup{
    colorlinks=true,
    citecolor=red,
    linkcolor=blue,
}

\newcommand{\PE}{\text{persistently exciting}}
\newcommand{\PoE}{\text{PoE}}

\DeclareMathOperator*{\E}{\mathbb{E}}

\newcommand{\vect}{\text{vec}}
\newcommand{\Ls}{\mathcal{L}}
\newcommand{\Xs}{{\mathcal{X}}}
\newcommand{\Ps}{{\mathcal{P}}}
\newcommand{\Ys}{{\mathcal{Y}}}

\newcommand{\R}{\mathbb{R}}
\newcommand{\I}{\mathbf{I}}

\newtheorem{definition}{Definition}[]
\newtheorem{problem}{Problem Statement}[]
\newtheorem{assumption}{Assumption}[]

\newtheorem{theorem}{Theorem}[]

\newtheorem{lemma}{Lemma}[]

\makeatletter
\newcommand\footnoteref[1]{\protected@xdef\@thefnmark{\ref{#1}}\@footnotemark}
\makeatother

\newcommand\blue[1]{{#1}}

\title{\LARGE \bf Improving Neural Network Robustness via Persistency of Excitation}

\author{Kaustubh Sridhar \qquad Oleg Sokolsky \qquad Insup Lee \qquad James Weimer
\thanks{The authors would like to acknowledge the support of the Army Research Office award number W911NF2010080.}
\thanks{The authors are with the PRECISE Center in the Departments of Electrical \& Systems Engineering and Computer \& Information Science, University of Pennsylvania, Philadelphia, PA, USA;
        {\tt\small \{ksridhar, sokolsky, lee, weimerj\}@seas.upenn.edu}.}%
\thanks{Code: \href{https://github.com/kaustubhsridhar/PoE-robustness}{\tt \footnotesize github.com/kaustubhsridhar/PoE-robustness}}
}

\begin{document}

\maketitle
\thispagestyle{empty}
\pagestyle{empty}

\begin{abstract} 
Improving adversarial robustness of neural networks remains a major challenge. Fundamentally, training a neural network via gradient descent is a parameter estimation problem. In adaptive control, maintaining persistency of excitation (PoE) is integral to ensuring convergence of parameter estimates in dynamical systems to their true values. We show that parameter estimation with gradient descent can be modeled as a sampling of an adaptive linear time-varying continuous system. Leveraging this model, and with inspiration from Model-Reference Adaptive Control (MRAC), we prove a sufficient condition to constrain gradient descent updates to reference persistently excited trajectories converging to the true parameters. The sufficient condition is achieved when the learning rate is less than the inverse of the Lipschitz constant of the gradient of loss function. We provide an efficient technique for estimating the corresponding Lipschitz constant in practice using extreme value theory. Our experimental results in both standard and adversarial training illustrate that networks trained with the PoE-motivated learning rate schedule have similar clean accuracy but are significantly more robust to adversarial attacks than models trained using current state-of-the-art heuristics.
\end{abstract}

\section{INTRODUCTION} \label{sec:intro}

Neural networks are vulnerable to adversarial examples \cite{szegedy2013intriguing} and most existing defenses are still highly susceptible to white box attacks \cite{carlini2017adversarial, athalye2018obfuscated} (where the adversary has full access to the network and its defense mechanism).



We believe that adversarial robustness can be improved by leveraging the fact that every neural network training process (standard or robust) is a parameter estimation problem \cite{nar2019persistency}, where the goal is to find the true parameters of a model. A model\footnote{Model and neural network are used interchangeably.} with its true parameters, \textit{i.e.}, the parameters of the true mapping from its input space to output space, always maps similar inputs to similar outputs \cite{nar2019persistency}. We posit (and em-\\pirically demonstrate) that this implies increased adversarial robustness for neural networks with their true parameters.

\blue{In system identification and adaptive control, Persistency of Excitation (\PoE{}) conditions \cite{sastry2011adaptive} are integral to robust estimation of true parameters. They restrict parameter estimation dynamics to exponentially-stable trajectories that ensure robust convergence to true values. Further, recent work \cite{nar2019persistency} analyzed neural network training and identified the lack of PoE in gradient descent (GD) as a major roadblock on the path to robustness. Thus, the main challenge addressed by this work is ensuring neural network training dynamics, and specifically gradient descent dynamics, is persistently excited and converges to the network's true parameters.}

\blue{Earlier attempts to characterize \PoE{} for GD were either impeded by a neural network's inherent nonlinearities \cite{lu1998robust, polycarpou1991identification} or limited to simple two layer networks and specific loss functions \cite{nar2019persistency}. In this work, we overcome the nonlinearity and complexity trap faced by \cite{nar2019persistency, lu1998robust, polycarpou1991identification} with the insight of modeling GD as a discretization of an adaptive continuous-time (CT) linear time-varying (LTV) system. We take inspiration from Model-Reference Adaptive Control (MRAC) \cite{sastry2011adaptive}, where adaptive control laws are chosen such that the system's dynamics emulate a reference system's dynamics, to propose the following two-step approach.}

\blue{First, we choose a reference family of persistently excited systems with a globally exponentially stable (GES) equilibrium at the unknown true parameters of the network. Then we prove sufficient conditions for consecutive updates of the discrete-time (DT) GD dynamics to lie on the exact discretization of a system from our reference family. Our novel two part approach theoretically guarantees convergence to the unknown true parameters of any model trained by minimizing a smooth loss with GD and empirically demonstrates increased robustness to adversarial attacks in stochastic gradient descent (SGD) based standard and adversarial training.}

\blue{
Our proven sufficient condition is equivalent to scaling a baseline learning rate schedule where the initial value is a function of the inverse of Lipschitz constant of the loss gradient.} To ensure a rigorous evaluation with minimal increase in model training time, we estimate this second-order Lipschitz constant with an inexpensive addition to the baseline model training procedure via extreme value theory \cite{wood1996estimation, weng2018evaluating}. To observe the utility of our \PE{} learning schedule, we apply it to standard training on MNIST \cite{lecun1998gradient}, CIFAR10, CIFAR100 \cite{cifar} datasets, and adversarial training on CIFAR10 dataset. We see an increase in adversarial accuracy of up to \blue{$15$}  points against a 20-step PGD adversary \cite{madry2017towards} with perturbation budget $\epsilon = \frac{1}{255}$ in standard training and an increase up to \blue{$0.7$} points in adversarial training on the competitive Autoattack benchmark \cite{croce2020reliable_autoattack} (with $\epsilon = \frac{8}{255}$) composed of both white-box and black-box attacks.

\subsection{Related Work}
\textbf{\PoE{} in Control Theory and Deep Learning.}  
\PoE{} has been thoroughly explored for CT LTV systems and is essential to robust parameter estimation in guaranteeing GES of parameter error dynamics which ensures convergence of estimated parameters to the true values \cite{sastry2011adaptive, srikant_PoE}. For learning-based system identification, early work \cite{rbf1, kurdila1995persistency} found \PoE{} conditions for Radial Basis Functions but emphasizes the difficulty in characterizing PoE conditions for general neural networks because of the nonlinearities in the models \cite{lu1998robust, polycarpou1991identification}.  

Recent seminal work in \cite{nar2019persistency} aims to tackle this challenging problem. Based on the premise of robust neural networks having bounded Lipschitz constants \cite{szegedy2013intriguing}, the authors derive sufficient richness conditions on the inputs to a two-layer network with ReLU activation functions trained with GD. However, their results are specific to a two-layer network initialized close to its true optima, particular loss functions and dependent on the gradient update rule. Moreover, these conditions do not scale to modern deep neural networks. To scale, they are forced to adopt an optimization trick to force noise (for \PoE{}) into each layer of a network. This trick can also be found in other robust learning approaches \cite{lecuyer2019certified, cohen2019certified}.


To avoid the issues faced in forward analysis by \cite{nar2019persistency}, we flip the problem around: we start with a well-characterized CT LTV family of persistently excited dynamics and then find sufficient conditions for GD updates to fit on the trajectories in this family. Our approach is only dependent on the gradient update rule and in practice, generalizes to all loss functions and scales to models that converge in training.

\textbf{Techniques for Robust Learning.} 
Adversarial training (AT), first introduced in \cite{goodfellow2014explaining}, is currently, the most effective defense to white-box attacks. AT requires solving a min-max optimization problem. The inner maximization problem is approximately solved with the PGD attack in PGD-AT \cite{madry2017towards}. A variant that modifies the inner maximization problem to tradeoff clean accuracy for robust accuracy was proposed in TRADES \cite{zhang2019theoretically}. \blue{Further improvement with additional unlabelled data (RST) \cite{carmon2019unlabeled} has increased robustness of models on the competitive AutoAttack benchmark}, \textit{a.k.a.} RobustBench (an ensemble of four white-box and black-box attacks with a single hyperparameter - perturbation budget $\epsilon = \frac{8}{255}$) \cite{croce2020reliable_autoattack}. Employing our PoE-motivated learning rate schedule further increases the robustness of models trained with the above state-of-the-art (SOTA) AT frameworks, thereby proving its importance as a force multiplier for any training algorithm.

\textbf{Estimating Lipschitz Constant of Loss Gradient.} Several works have studied neural network Lipschitz constant estimation (\textit{e.g.} \cite{fazlyab2019efficient, weng2018evaluating}) but here, we are concerned with the Lipschitz constant of loss gradient (denoted $\Ls$). In \cite{herrera2020estimating}, approximate upper bounds were derived for $\Ls$ but to our best knowledge, no efficient estimation method has been previously proposed and implemented in practical SGD. In this work, we apply an extreme value theory approach \cite{wood1996estimation, weng2018evaluating} and estimate $\Ls$ in both standard and adversarial training.

\subsection{Contributions}
1) \blue{We propose extending PoE, with inspiration from MRAC, to neural network training to obtain sufficient conditions for convergence of GD dynamics for any model to its true parameters. Our insight into modeling GD as a sampling of an adaptive CT LTV system is vital to generalizing beyond the simple 2-layer network and certain loss functions in \cite{nar2019persistency}.}

2) \blue{We present an efficient implementation strategy in practical SGD training, based on extreme value theory, to obtain an estimate of the initial learning rate in a learning schedule for \PoE{}. We also detail a simple heuristic to tune batch size to satisfy a principal assumption in our derivation.}

3) We demonstrate the effectiveness of our approach in standard training with SGD on MNIST, CIFAR10, and CIFAR100 (up to \blue{$15$} points accuracy increase on 20-step, $\epsilon = 1/{255}$ PGD attack) \& with various SOTA adversarial trained CIFAR10 models on the competitive Autoattack benchmark (universal improvements of up to \blue{$0.7$} points with $\epsilon = 8/{255}$).

\section{PROBLEM FORMULATION} \label{sec:prob}
In this Section, we mention some preliminaries from \blue{adaptive control} \& GD and then formally state our problem.\\
\blue{\textbf{Adaptive Control Preliminaries:} For a continuous-time (CT) linear time-varying (LTV) system given by},
\begin{align}
    \Dot{z}(t) = -\Phi(t) \Phi(t)^T z(t), \;\; t \geq 0 \label{CT_LTV_sys}
\end{align}
with $z(t) \in \R^d, \; \Phi(t) \in \R^{d \times p}$, \blue{\PoE{} is defined as follows.}
\begin{definition}[\textbf{\PoE{}} \cite{sastry2011adaptive}] The signal $\Phi(t): \R^{\geq 0} \to \R^{d \times p}$ is \PE{} 
if there exists $\mu_1, \mu_2, T_0 > 0$ such that, 
\begin{align}
    \mu_2 \I \geq \int_{t}^{t+T_0} \Phi(s) \Phi(s)^{\top} ds \geq \mu_1 \I
\end{align}
where $\I$ is the $d \times d$ Identity matrix. \\
\textup{\PoE{} and GES are connected via the following lemma.}
\label{PoE_basic_defn}
\end{definition}

\begin{lemma}[\textbf{\PoE{} and GES} {\cite[Theorem 2.5.1]{sastry2011adaptive}}] 
If $\Phi(t)$ is piece-wise continuous and \PE{}, then system \eqref{CT_LTV_sys} is GES.
\label{GES_PE_lemma}
\end{lemma}
\vspace{-0.5em}
Lemma \ref{GES_PE_lemma} ensures GES convergence of states on \eqref{CT_LTV_sys} to their equilibrium and informs our definition of the persistently exciting reference family for GD updates to track.\\
\textbf{GD Preliminaries.}
We represent feature space with $\Xs$, label space with $\Ys$ and model with parameters $\Theta \in \Ps$ given by $h_{\Theta} : \Xs \to \Ys$. In the theoretical part of this work, we focus on model training with GD wherein, we represent the training data with $(X, Y) \in \Xs \times \Ys$ and the loss function minimized with $L: \Ys \times \Ys \to \R$. We denote the vectorized version of parameters as $\theta = \vect (\Theta) \in \R^d$, vectorized loss gradients as $\nabla l(\theta) = \vect \left( \nabla L(h_{\Theta}(X), Y) \right) \in \R^d$, and learning rate as $\eta$. Now, we write the vectorized GD update step below.

\begin{definition}[\textbf{Vectorized GD Update}] \label{GD_defn} The vectorized form of the $k$th update step in GD for training a model $h_{\Theta}$ on training data $(X, Y)$ by minimizing loss $L(h_{\Theta}(X), Y)$ with learning rate $\eta^k$ is given by
\begin{align}
    \theta^{k+1} = \theta^k - \eta^k \nabla l (\theta^k), \;\; k = 1,2,\ldots. \label{GD}
\end{align} 
\end{definition}
The above definition casts GD into a DT nonlinear time-varying (NLTV) system. 
Analyzing this system for a particular loss function $l$ and model architecture $h_{\theta}$ is intractable for increasingly larger models trained by minimizing custom loss functions. We propose a \blue{bottom-up} solution for this problem.

\blue{First, we choose a family of persistently excited CT dynamics such that all of its members converge exponentially-fast to the unknown true parameters (denoted $\theta^*$) of a model. In this work, we conjecture (and empirically demonstrate in Sections \ref{sec:implement}, \ref{sec:results}) that the true parameters coincide with the maximal $\epsilon$-robust optimum (where an $\epsilon$-robust optimum provides a model with the same output for all inputs in a ball of perturbation size $\epsilon$ around any input in domain $\Xs$).}

Second, we \blue{find sufficient conditions} to constrain consecutive states of the DT NLTV system in \eqref{GD} to lie on \blue{discretized trajectories from the aforementioned family. Through these two steps we obtain sufficient conditions for convergence of GD updates to the true parameters of the model. We tackle the first of the two steps below by choosing the following family of dynamics with} each member of the family, for $k\geq 1$, starting at the $k$th GD update ($\theta^k$).

\begin{definition}[\textbf{Reference Family of Persistently Exciting CT Systems}] 
The reference family of \PE{} CT systems, with GES equilibrium at $\theta^*$, that governs the evolution of a state vector $\Gamma(t) \in \R^d, \; t \geq k$ with initial value $\Gamma(k) = \theta^k$, is given by 
\begin{align}
\Dot{\Gamma}(t) &= - \Phi(t) \Phi(t)^{\top} (\Gamma(t) - \theta^*)
\label{PE_traj}
\end{align}
where $\Phi(t) = \Phi^k \; \forall \; t \in [k, k+1)$ is piece-wise constant matrix $\in \R^{d \times p}$ and $\Phi^k {\Phi^k}^{\top}$ is full rank $\forall \; k$.
\label{family_PE_defn}
\end{definition}
Equation \eqref{PE_traj} is a form of the CT LTV system in \eqref{CT_LTV_sys} with an equilibrium at $\theta^*$. Our choice of $\Phi(t)$ in the above definition ensures that $\Phi(t)$ is \PE{} and consequently that the system in \eqref{CT_LTV_sys} has GES equilibrium (Proved in Section \ref{sec:main}). 
Finally, with the requisite family defined, we formally state the problem considered by this paper below.

\begin{problem}[\textbf{\PoE{} of GD}]
We aim to find sufficient conditions for every pair of consecutive $k$th-step GD updates $\theta^k, \theta^{k+1}$ (Definition \ref{GD_defn}) to lie on a \blue{discretized} trajectory from the reference persistently excited \blue{CT} family in Definition \ref{family_PE_defn}.

\textup{\blue{We remark that any mention of '\PoE{} of GD' in this work denotes the above desired property of the GD dynamics.}}
\label{ps}
\end{problem}
\section{MAIN THEORETICAL RESULTS ON LEARNING RATES FOR PoE OF GD} \label{sec:main}


In this section, we present our main theoretical result in Theorem \ref{main_theorem} which proposes an upper bound on learning rates $\eta^k, \; k \geq 1$ to accomplish our problem statement. Also, we discuss the proven sufficient condition and its connections to convex optimization \& constant learning rate training. Before stating Theorem \ref{main_theorem}, we present our assumptions on the $\Ls-$smoothness of loss and the acuteness of descent directions below.

\begin{assumption}[\textbf{$\mathbf{\Ls-}$Smooth Loss Function}]
The loss function is $\Ls-$smooth if its gradient $\nabla l: \R^d \to \R^d$ is $\Ls-$Lipschitz, \textit{i.e.} there exists a constant $\Ls > 0$ such that
\begin{align}
    \forall \; \theta_1, \theta_2 \in \R^d, \;\; \|\nabla l(\theta_2) - \nabla l(\theta_1)\|_2 \leq \Ls \|\theta_2 - \theta_1\|_2.
\end{align}
\textup{$\Ls$ is also called the second-order Lipschitz constant. $\Ls$-smoothness is a commonly recurring assumption in optimization theory \cite{nesterov2018lectures, bertsekas1997nonlinear}. In practice, standard and adversarial losses are smooth for certain models \cite{li2017visualizing, wu2020adversarial} and not others. Yet, in Section \ref{sec:results}, our approach results in increased robustness for a wide variety of architectures.}
\label{main_assumption}
\end{assumption}

\begin{assumption}[\blue{\textbf{Acuteness of descent directions}}]
\blue{The angle between the $k$th descent direction and the next true descent direction (from $(k+1)$th update to true parameters) is acute.}
\begin{align}
    \blue{\textit{i.e.} \;\; (\theta^k - \theta^{k+1})^{\top}(\theta^{k+1} - \theta^{*}) \geq 0 \label{eq:acute_assumption}}
\end{align}
\textup{\blue{Assumption \ref{acute_assumption} states that the local gradient and true gradient are acute, an intuitive property of GD with training data that is representative of the population. Further, we monitor this assumption in our experiments in Section \ref{sec:implement} and observe that it is indeed satisfied throughout model training with large batch SGD and with GD (full-batch SGD)}. Thus, leveraging Assumptions \ref{main_assumption}, \ref{acute_assumption}, we state the main theorem below.}
\label{acute_assumption}
\end{assumption}
\begin{theorem}[\textbf{Sufficient Conditions for \PoE{} of GD}]
\blue{Consider} a model trained via GD with a learning rate schedule given by $\left(\eta^k\right)_{k\geq 1}$ by minimizing a $\Ls-$smooth loss (Assumption \ref{main_assumption}) \blue{and satisfying Assumption \ref{acute_assumption}}. We have PoE of GD \blue{and hence convergence of GD updates to the model's true parameters} if $\eta^k < 1/\Ls$ for all $k$.
\label{main_theorem}
\end{theorem}%
\begin{proof}
We begin with our sufficient condition: $\eta^k <  \frac{1}{\mathcal{L}}$ which can be rewritten as $\mathcal{L} <  \frac{1}{\eta^k}$ such that,
\begin{align}
    \frac{\|\nabla l(\theta^k) - \nabla l(\theta^*) \|_2}{\|\theta^k - \theta^*\|_2} &< \frac{1}{\eta^k} \label{desired_from_PE}
\end{align}
Observing the gradient at the optima is zero in \eqref{desired_from_PE}, i.e., $\nabla l(\theta^*)=0$, we have, $\frac{\|\nabla l(\theta^k) - 0 \|_2}{\|\theta^k - \theta^*\|_2} < \frac{1}{\eta^k} \iff \eta^k \|\nabla l(\theta^k)\|_2 < \|\theta^k - \theta^*\|_2$. By substituting \eqref{GD}, we have,
\begin{align}
    \|\theta^{k} - \theta^{k+1}\|_2 &< \|\theta^k - \theta^*\|_2 \; \nonumber
    \\
    \iff \|U^{\top} (\theta^k - \theta^{k+1})\|_2 &< \|U^{\top} (\theta^k - \theta^{*})\|_2 \; \forall \; U \in \blue{SO(d)} \nonumber
\end{align}
where $\blue{SO(d)}$ is the set of orthonormal rotation matrices in $d$-dimensions (since rotated vectors maintain their magnitudes). Now, choosing $U = [v_1, v_2, v_3, ..., v_d], \; v_i^{\top} v_j = 0, \; {\|v_i\|}_2 = 1$ where (for infitesimally small $\delta>0$),
\blue{\begin{align}
    v_1 &= \frac{(\theta^k - \theta^{k+1}) - \delta (\theta^k - \theta^{*})}{\|(\theta^k - \theta^{k+1}) - \delta (\theta^k - \theta^{*})\|_2} \label{v1_formula}\\
    v_2^{\top} v_1 =0 &\implies  v_2^{\top} (\theta^k - \theta^{k+1}) - \delta v_2^{\top} (\theta^k - \theta^{*}) = 0 \label{v1_v2_reln}
\end{align}}
\blue{then we can write, $(\theta^k - \theta^{k+1}) = a_1 v_1 + a_2 v_2$ and $(\theta^k - \theta^*) = b_1 v_1 + b_2 v_2$ for constants $a_1, a_2, b_1, b_2$ as follows.}
\blue{\begin{align}
    a_1 &= v_1^{\top} (\theta^k - \theta^{k+1}) \nonumber 
    \\
    &= \frac{\|\theta^k - \theta^{k+1}\|_2^2 - \delta (\theta^k - \theta^{k+1})^{\top} (\theta^k - \theta^{*})}{\|(\theta^k - \theta^{k+1}) - \delta (\theta^k - \theta^{*})\|_2} \label{a_1}
    \\
    b_1 &= v_1^{\top} (\theta^k - \theta^*) \nonumber 
    \\
    &= \frac{(\theta^k - \theta^{k+1})^{\top} (\theta^k - \theta^{*}) -\delta \|\theta^k - \theta^{*}\|_2^2}{\|(\theta^k - \theta^{k+1}) - \delta (\theta^k - \theta^{*})\|_2} \label{b_1}
    \\
    a_2 &= v_2^{\top} (\theta^k - \theta^{k+1}) = \delta v_2^{\top} (\theta^k - \theta^{*}) \;\; (\text{via}\;\eqref{v1_v2_reln}) \label{a_2}
    \\
    b_2 &= v_2^{\top} (\theta^k - \theta^{*}) \label{b_2}
\end{align}}
\blue{From Assumption \ref{acute_assumption}, we have}
\blue{\begin{align}
    \frac{(\theta^k - \theta^{k+1})^{\top}(\theta^k - \theta^*)}{\|\theta^k - \theta^{k+1}\|^2_2} &= 1 + \frac{(\theta^k - \theta^{k+1})^{\top} (\theta^{k+1} -  \theta^{*} )}{\|\theta^k - \theta^{k+1}\|^2_2} \nonumber
    \\
    &\geq 1 \;\;\;(\text{from}\; \eqref{eq:acute_assumption}) \nonumber
    \\
    \implies (\theta^k - \theta^{k+1})^{\top}&(\theta^k - \theta^*) \geq \|\theta^k - \theta^{k+1}\|^2_2. \label{new_acute}
\end{align}}
\blue{Therefore, in the limit of $\delta \to 0$, we have $b_1 \geq a_1 > 0$ (from \eqref{a_1}, \eqref{b_1},  \eqref{new_acute}) and $b_2 > a_2 \to 0^{+}$. The latter is because as $\delta \to 0$, we have $v_1$ lying along $(\theta^k - \theta^{k+1})$ which means, from \eqref{new_acute}, $v_1$ and $(\theta^k - \theta^{*})$ are acute. This in turn implies $v_2$ and $(\theta^k - \theta^{*})$ are acute and hence from \eqref{b_2}, $b_2 > 0$ and from \eqref{a_2}, $a_2 \to 0$ from the positive side.}

\blue{
Thus, for the above choice of $U$, in the limit of $\delta \to 0$, we have,
$
    U^{\top} (\theta^k - \theta^{k+1}) 
    = [v_1^{\top}, v_2^{\top}, \ldots, v_d^{\top}]^{\top} (a_1 v_1 + a_2 v_2) 
    = [a_1, a_2, 0, \ldots, 0 ]^{\top}
    \leq [ b_1, b_2, 0, \ldots, 0 ]^{\top} 
    = [ v_1^{\top}, v_2^{\top}, \ldots, v_d^{\top}]^{\top} (b_1 v_1 + b_2 v_2)
    = U^{\top} (\theta^k - \theta^{*})
$.}

\blue{
Continuing in the limit of $\delta \to 0$} and choosing $\Sigma = \text{diag}\left( \frac{a_1}{b_1}, \frac{a_2}{b_2}, c_3, \ldots, c_d \right)$ where $0 < c_i < 1$ for $i=3,\ldots, d$ (note $0 < \Sigma < \I$), we can scale down the right hand side vector above to match the left hand side vector as follows,
\begin{align}
    U^{\top} (\theta^k - \theta^{k+1})&= \Sigma U^{\top} (\theta^k - \theta^{*}) \nonumber
    \\
    \iff (\theta^k - \theta^{k+1})&= U \Sigma U^{\top} (\theta^k - \theta^{*}) \nonumber
    \\
    \iff (\theta^k - \theta^{k+1})&= (\mathbf{I} - e^{-\Phi^k {\Phi^k}^{\top}}) (\theta^k - \theta^{*}) \label{to_rewrite}
\end{align}
Since $\Phi^k {\Phi^k}^{\top}$ is full rank and we observe $(\mathbf{I} - e^{-\Phi^k {\Phi^k}^{\top}}) = V D V^{\top}$ with $V \in \blue{SO(d)}$ and diagonal $0< D < \mathbf{I}$, we can choose $\Phi^k$ such that $V=U$ and $D = \Sigma$. Rewriting \eqref{to_rewrite},
\begin{align}
    \theta^{k+1} - \theta^* &= e^{-\Phi^k {\Phi^k}^{\top}} (\theta^k - \theta^*). \label{discrete_fit}
\end{align}
Since \eqref{discrete_fit} is equivalent to the discretization of the CT dynamics of system \eqref{PE_traj} in time interval $[k,k+1]$ with $\Gamma(k+1)=\theta^{k+1}$ and initial value $\Gamma(k)=\theta^{k}$, we have proven that $\theta^k, \theta^{k+1}$ lie on the discretized trajectory of said system from the reference family of Definition \ref{family_PE_defn}. 

Finally, since $\Phi^k {\Phi^k}^{\top}$ is full rank and $\Phi^k {\Phi^k}^{\top}$ is positive definite, the system from \eqref{PE_traj} in time interval $[k,k+1]$ given by $\dot{\Gamma}(t) = - \Phi^k {\Phi^k}^{\top} (\Gamma(t) - \theta^*), t \geq k$
is persistently excited (since for any $T>0$, we have $\int_{t}^{t+T} \Phi^k {\Phi^k}^{\top} ds = \Phi^k {\Phi^k}^{\top} T$ and $0 < \lambda_{\min} T \leq \Phi^k {\Phi^k}^{\top} T \leq \lambda_{\max} T$ where $\lambda_{\min}, \lambda_{\max}$ are the minimum and maximum eigenvalues of positive definite $\Phi^k {\Phi^k}^{\top}$) and has GES equilibrium at true optimum $\theta^*$. 
\end{proof}

\blue{Next, we provide some discussions on the main result.}
\\
\textbf{Remark 1. On sufficient conditions for PoE and the conservative upper bound of $1/\Ls$.} 
It is worth noting that the proof has two sufficient conditions: full rankness of $\Phi^k {\Phi^k}^{\top}$ which is sufficient but not necessary for \PoE{} and $\eta^k < 1/\Ls$ which is sufficient but not necessary for Inequality \eqref{desired_from_PE} to hold. With these 2 sufficient conditions, our upper bound $1/\Ls$ is conservative and values greater than it may also ensure \PoE{}. In fact, inspired by empirical successes in Sections \ref{sec:implement}, \ref{sec:results} we later conjecture that an upper bound of $2/\Ls$ may ensure \PoE{}. Further, a necessary \& sufficient condition for \PoE{} instead of the first sufficient condition above may provide a larger upper bound and is an interesting open problem.\\
\textbf{Remark 2. On the connection to convex optimization.} \\
In GD on a $\Ls-$smooth and convex loss function, a learning rate choice of $\eta^k \leq \frac{1}{\Ls}$ guarantees monotonic progress to the minima \cite{nesterov2018lectures}. Our similar result is expected because every \PE{} trajectory, on which states converge exponentially fast to minima $\theta^*$, is a convex shortcut from $\theta^k$ to $\theta^*$ through $\theta^{k+1}$ that may/may not lie on the loss surface. 
This relationship provides an interesting intuition for our approach and strengthens its validity.
\\
\textbf{Remark 3. On training with a constant learning rate and GD vs SGD.} In GD on any $\Ls-$smooth differentiable loss function with a fixed learning rate $\eta$, the algorithm converges to local minima if $\eta < \frac{2}{L}$ \cite{bertsekas1997nonlinear}. Thus, if a model converges via GD with a fixed learning rate, halving it should be adequate for \PoE{}. Unfortunately the simplicity of dividing by 2 does not always work in practice because modern neural network training algorithms, faced with GPU constraints, use SGD rather than GD and learning schedules rather than a constant learning rate for which this holds. We discuss this gap ahead.
\section{IMPLEMENTATION IN SGD TRAINING} \label{sec:implement}

In practice, models are trained with SGD and its variants \cite{nesterov2018lectures} rather than GD where a slowly-decaying learning rate schedule (such as a sequence obeying $\sum_k \eta^k = \infty, \; \sum_k {\eta^k}^2 < \infty$ \cite{bottou2018optimization}) is often necessary for training to converge in the first place. \blue{Thus, in this Section, we present an implementation strategy for SGD based training that satisfies Theorem \ref{main_theorem}, Assumption \ref{acute_assumption} and actually converges.}

\subsection{PoE-motivated learning rate schedules for SGD} \label{subsec:PoE_schedules}
We assume a baseline learning rate schedule, $(\gamma^k)_{k \geq 1}$ with $\gamma^{k}\leq \gamma^1$ $\forall$ $k$, that ensures convergence to a local optima. Our \textit{PoE-motivated learning rate schedule} starts at $\eta^1 = \frac{1}{\Ls_{\text{est}}}$ and is subsequently scaled in the same way as the baseline schedule, \textit{i.e.} $\eta^k = \eta^1 \frac{\gamma^k}{\gamma^1}$ $\forall$ $k \geq 2$. Our algorithm for obtaining $\Ls_{\text{est}}$ (see Section \ref{subsec:Lips_esti}) always provides an estimate larger than the true value \cite{wood1996estimation, weng2018evaluating} ensuring that $\eta^1 = \frac{1}{\Ls_{\text{est}}} < \frac{1}{\Ls_{\text{true}}}$ and since $\eta^k \leq \eta^1$ $\forall$ $k \geq 2$, Theorem \ref{main_theorem} is satisfied. Further, by following a similar annealing cycle as the baseline schedule, we have convergence in practice. Lastly, following Remark 1, we analyze another schedule, \textit{a.k.a.} \textit{largest convergent schedule}, where $\eta^1=2/\Ls_{\text{est}}$ and we similarly scale subsequent values $\eta^k = \eta^1 \frac{\gamma^k}{\gamma^1}$ $\forall$ $k \geq 2$. This too ensures convergence in practice and we later conjecture in Section \ref{sec:conclusion} that it also leads to \PoE{}. Figure \ref{example_lrs} shows an example of these schedules for typical annealing strategies.

Our choice of initial learn rate $\eta^1 = \frac{1}{\Ls_{\text{est}}}$ which is close to Theorem \ref{main_theorem}'s upper bound stems from an experimental analysis of adversarial accuracy versus learning rate. We trained LeNet5 models \cite{lecun1998gradient} (with ReLU/Tanh activations) with various constant learning rates for 10 epochs via SGD on MNIST. Each trained model is evaluated against a 40-step PGD attack with $\epsilon=0.3$. Plotting the clean and PGD attack accuracy in Figure \ref{mnist}, we see that PGD attack accuracy peaks near the largest learning rate at which the model converges to a local optima (\textit{i.e.} when clean accuracy is close to 100\%), theoretically given by $\frac{2}{\Ls}$ \cite{bertsekas1997nonlinear}. Also, at half this point (our upper bound of $\frac{1}{\Ls}$), we notice PGD attack accuracy is still high but drops immediately on the left. This drop can be explained by the ill-conditioning of $||\nabla l (\theta^k)||/||\theta^k-\theta^*||$ term in \eqref{desired_from_PE} where the denominator is small for small learning rates. Moreover, since it is hard to predict the exact drop point, we stay close to the upper bound and use $\frac{1}{\Ls_{\text{est}}}$ \& $\frac{2}{\Ls_{\text{est}}}$. 

\begin{figure}[t]
\centering
\vspace{0.5em}
\includegraphics[width=\linewidth]{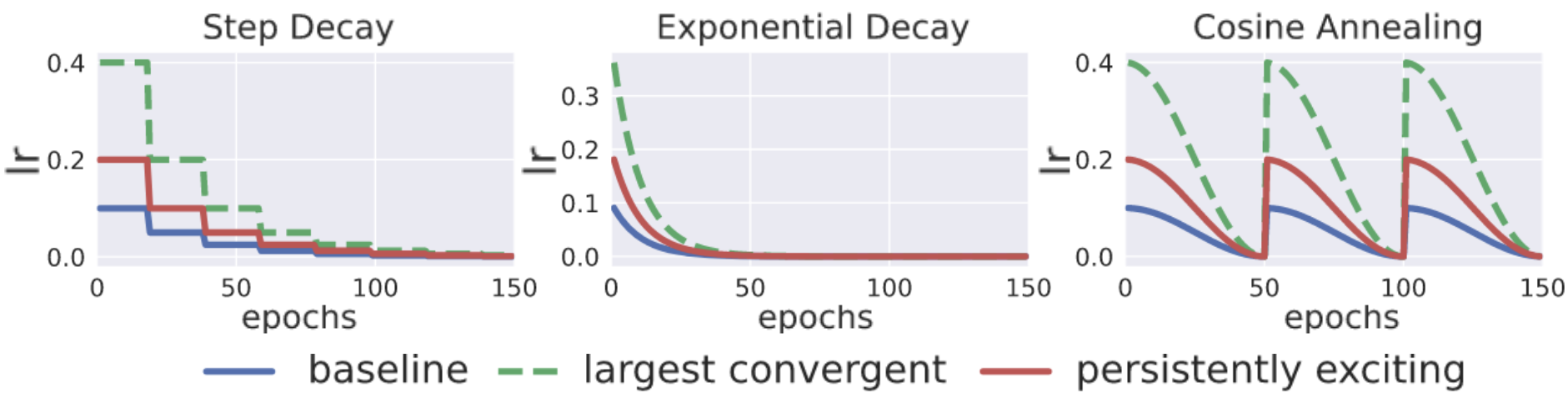}
\vspace{-1.5em}
\caption{An \textit{e.g.} of baseline ($\eta^1=0.1$), largest convergent ($\eta^1=2/\Ls$) \& PoE-motivated ($\eta^1=1/\Ls$) learning rate (lr) schedules for step decay, exponential decay \& cosine annealing strategies.} 
\label{example_lrs}
\vspace{-0.75em}
\end{figure}

\begin{figure}[t]
\centering
\includegraphics[width=.47\linewidth]{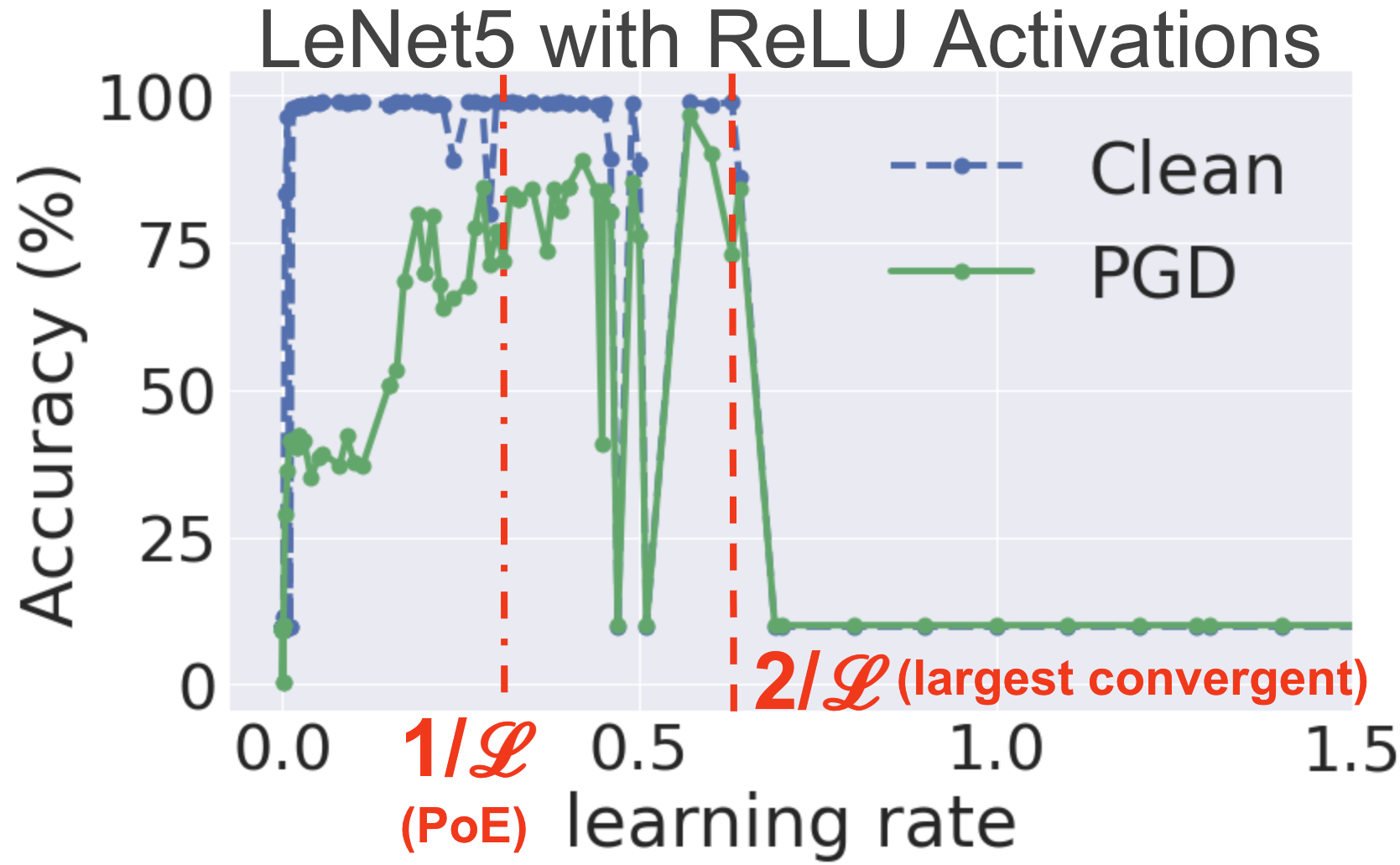}
\hfill
\includegraphics[width=.51\linewidth]{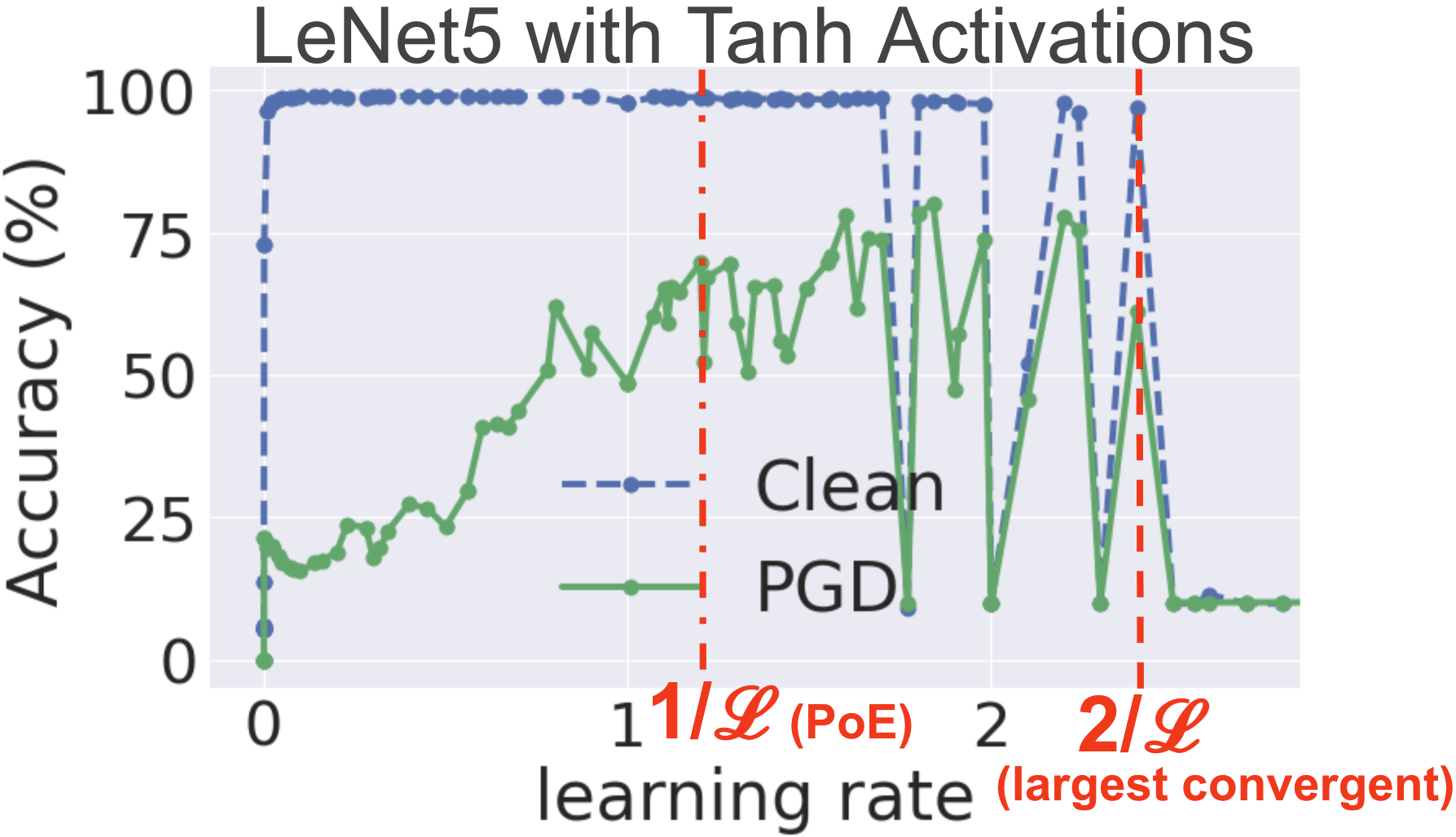}
\vspace{-1.5em}
\caption{Accuracy on Clean and PGD attacked MNIST validation set for a LeNet5 model (with ReLU [left] and Tanh [right] activations) vs constant learning rate (lr) used in training. The largest convergent and the \PoE{}-motivated lr's have higher PGD accuracy than baseline lr $= 0.1$.} \label{mnist}
\vspace{-1.5em}
\end{figure}

\vspace{-0.25em}
\subsection{Estimation of Certified Lipschitz Constant \texorpdfstring{$\Ls_{\text{est}}$}{L}} \label{subsec:Lips_esti}
Inspired by \cite{wood1996estimation, weng2018evaluating}, we estimate $\Ls$ with three steps:\\
\textbf{(1)} We collect average loss gradient and model parameters after each epoch in baseline training (\textit{i.e.} with the baseline schedule). They're denoted by $(\nabla l(\theta^{i})$, $\theta^i)_{1, \ldots, N_{\text{epochs}}}$.\\
\textbf{(2)} We estimate a Lipschitz constant by sampling $N$ points, computing $N/2$ slopes between consecutive pairs as 
$
            s_i = \frac{||\nabla l(\theta^{i+1}) - \nabla l (\theta^i)||_2}{||\theta^{i+1} - \theta^i||_2}
$, $i=1, 3, 5, ..., N$ and finding the maximum, $l = \max \{s_1, s_3, s_5, ..., s_{N}\}$. We repeat this $M$ times and applying the Fisher–Tippett–Gnedenko theorem \cite{wood1996estimation}, fit a 3 parameter (shape, location, scale) reverse Weibull distribution to $\{l_1, ..., l_M\}$ given an initial shape value. The fitted scale parameter is the desired estimate of Lipschitz constant.\\
\textbf{(3)} We certify our estimated Lipschitz constant by iterating between Step (2) and a Kolmogorov–Smirnov (K-S) test to test that our samples $\{l_1, ..., l_M\}$ are drawn from a reverse Weibull distribution with the Step (2)'s fitted parameters. Out of various p-values obtained, we choose the Lipschitz constant (\textit{i.e.} scale parameter) with the highest p-value. However, a question remains, how are hyperparameters $M, N$ tuned?\\
\textbf{A heuristic for hyper-parameter (M, N) tuning.} We repeat steps (2), (3) for different $M$, $N$ and choose the Lipschitz constant from the case when atleast one p-value is both larger and smaller than a mid-to-large significance value $\alpha=0.4$ to $0.6$. This heuristic works well in practice. Using the above, we estimate $\Ls_{\text{est}}$ for ResNet20 standard training (\textit{i.e.} minimizing cross-entropy loss on clean images) with $\alpha=0.55, M=200, N=100$. Plotting clean \& PGD accuracy in Figure \ref{opt1_only_plot}, we observe that models trained with the PoE-motivated and largest convergent schedules are more robust than the baseline while matching its clean accuracy.\\
\textbf{Comparison to grid search: }By having to train a baseline model to estimate $\Ls$ before training with a \PoE{} schedule, we have a $2\times$ increase in training time. Grid search, on the other hand, is not theoretically motivated and is unlikely to obtain an optimal learning rate schedule in just 2 training rounds. Thus, our PoE-motivated approach is the clear winner.

\begin{figure}[t]
	\centering
	\vspace{0.5em}
	\includegraphics[width=\linewidth]{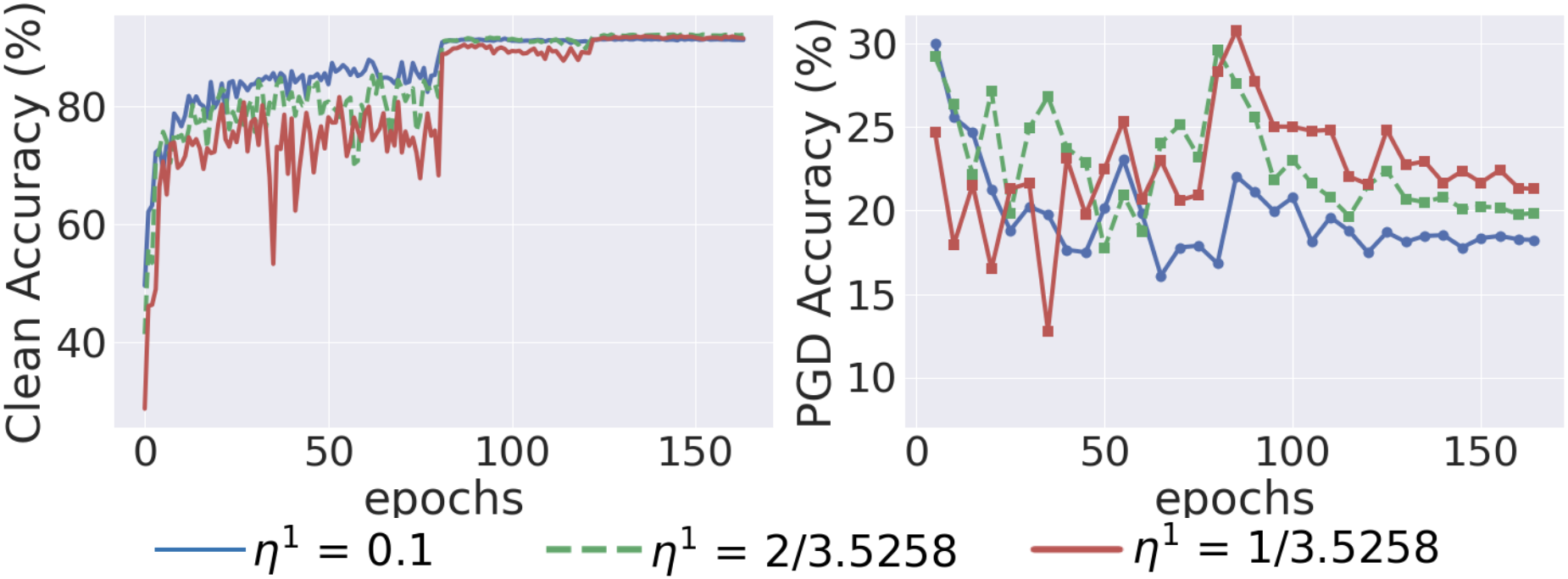}
	\vspace{-.25em}
	\begin{minipage}{0.49\linewidth}
	\centering
	\footnotesize{(a) Clean accuracy vs epochs}
	\end{minipage}
	\begin{minipage}{0.49\linewidth}
	\centering
	\footnotesize{(b) PGD attack accuracy vs epochs}
	\end{minipage}
	\vspace{-1.75em}
	\caption{Lipschitz constant estimated with extreme value theory for ResNet 20 standard training is $3.5258$. Training with \PoE{}-motivated ($\eta^1=1/{\Ls_{\text{est}}}$) and largest convergent ($\eta^1=2/{\Ls_{\text{est}}}$) schedules consistently increases PGD attack accuracy while matching clean accuracy of baseline (epochs 80-end).}
	\label{opt1_only_plot}
	\vspace{-0.75em}
\end{figure}

\subsection{\blue{Batch size selection for Assumption \ref{acute_assumption}}} \label{subsec:batch}
\blue{Following the discussion post Assumption \ref{acute_assumption} and with the parameters saved every epoch in Section \ref{subsec:Lips_esti}, we monitor Inequality \eqref{eq:acute_assumption} in ResNet20, 50 standard training on CIFAR10 with $\theta^*$ set to the final parameters. Plotting the fraction of total epochs in which it holds against batch size in Figure \ref{fig:acute_assumption_monitor}, we see that it is indeed satisfied for all epochs when trained with large batch sizes but faces GPU memory constraints \& clean accuracy degradation \cite{keskar2016large_batch} with said large batch sizes. For the tradeoff, we find that a simple heuristic of starting with the smallest batch size at which Assumption \ref{acute_assumption} holds \& decreasing until a clean accuracy threshold (here, $0.8$ on CIFAR10, $0.6$ on CIFAR100) is reached, works in practice.}

\begin{figure}[t]
	\centering
	\includegraphics[width=0.50\linewidth]{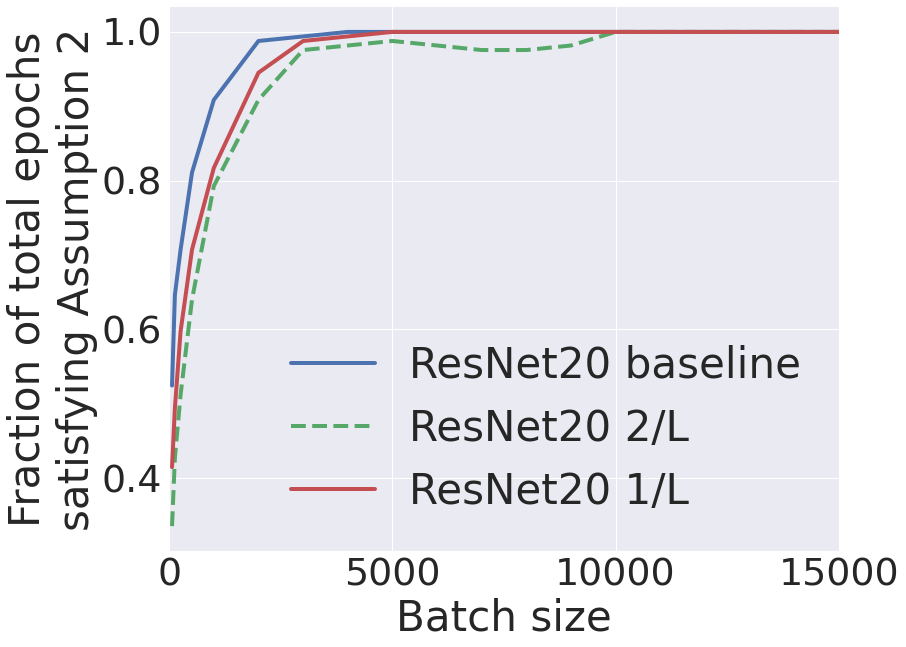}
    \hfill 
	\includegraphics[width=0.48\linewidth]{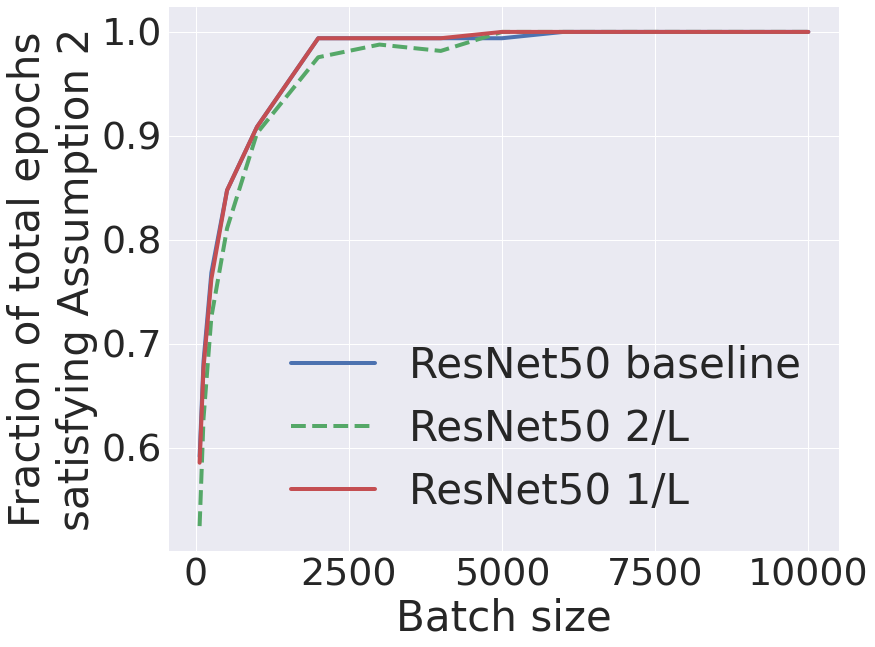}
	\vspace{-1.75em}
	\caption{\blue{Fraction of total epochs (164) where Assumption \ref{acute_assumption} holds vs batch\\size for ResNet20 [left], ResNet50 [right] standard training on CIFAR10 with all three schedules. The Assumption always holds for large batch sizes.}}
	\label{fig:acute_assumption_monitor}
	\vspace{-1.5em}
\end{figure}
\section{EXPERIMENTAL RESULTS} \label{sec:results}
\begin{table*}[t]
\centering
\vspace{0.60em}
\caption{\blue{Standard trained models on CIFAR10 and CIFAR100 evaluated on 20-step PGD attack with $\epsilon = 1 / 255$.
\\
\vspace{-0.35em}
Bold and underlined numbers denote the best and 2nd best PGD Attack accuracy in each row.}}
\vspace{-0.5em}
\begin{tabular}{c @{\hspace{0.9\tabcolsep}} c|cc|cc|cc|cc}
\hline
Dataset                   & Model          & \multicolumn{2}{c|}{baseline $\eta^1$} & \multicolumn{2}{c|}{\begin{tabular}[c]{@{}c@{}}PoE-motivated \\ $\eta^1=1/\Ls_{\text{est}}$ (Ours)\end{tabular}} & \multicolumn{2}{c|}{\begin{tabular}[c]{@{}c@{}}Largest convergent \\ $\eta^1=2/\Ls_{\text{est}}$ (Ours)\end{tabular}} & \multicolumn{2}{c}{Section \ref{subsec:Lips_esti}, \ref{subsec:batch} Parameters} \\ \hline
                          &                & Clean              & PGD Attack               & Clean                                              & PGD Attack                                                         & Clean                                                         & PGD Attack                                                                   & $\Ls_{\text{est}}$, (M, N), epochs        & Batch       \\ \hline
CIFAR                   & ResNet20      &   81.55 \scriptsize{$\pm$ 2.6} & 24.69 \scriptsize{$\pm$ 2.3} & 83.03 \scriptsize{$\pm$ 0.2} & {\ul 28.72 \scriptsize{$\pm$ 1.3}} & 81.92 \scriptsize{$\pm$ 1.3} & \textbf{31.9 \scriptsize{$\pm$ 3.1}}                                                          & 3.526, (200, 100), 164              & 5000        \\
    10                      & ResNet50      & 84.44 \scriptsize{$\pm$ 2.0} & 24.77 \scriptsize{$\pm$ 2.5} & 84.29 \scriptsize{$\pm$ 2.0} & {\ul 27.35 \scriptsize{$\pm$ 2.4}} & 84.24 \scriptsize{$\pm$ 2.7} & \textbf{35.82 \scriptsize{$\pm$ 2.8}}                                                                    & 10.79, (200, 164), 164              & 2000        \\
                          & ResNet110     & 83.77 \scriptsize{$\pm$ 0.5} & 28.81 \scriptsize{$\pm$ 2.5} & 82.63 \scriptsize{$\pm$ 0.9} & \textbf{39.08 \scriptsize{$\pm$ 2.5}} & 84.62 \scriptsize{$\pm$ 0.4} & {\ul 34.22 \scriptsize{$\pm$ 2.8}}                                                           & 11.85, (200, 164), 164              & 1000        \\
                          & DenseNet40    & 82.97 \scriptsize{$\pm$ 2.8} & 12.11 \scriptsize{$\pm$ 1.1} & 85.75 \scriptsize{$\pm$ 2.4} & {\ul 14.81 \scriptsize{$\pm$ 0.8}} & 87.92 \scriptsize{$\pm$ 2.1} & \textbf{15.97 \scriptsize{$\pm$ 1.6}}                                                            & 5.429, (100, 100), 300              & 2000        \\ \hline
CIFAR & ResNet50      & 63.46 \scriptsize{$\pm$ 4.4} & 6.9 \scriptsize{$\pm$ 1.1} & 63.59 \scriptsize{$\pm$ 4.1} & {\ul 7.56 \scriptsize{$\pm$ 1.1}} & 65.01 \scriptsize{$\pm$ 4.0} & \textbf{8.51 \scriptsize{$\pm$ 1.4}}                                                         & 8.75, (200, 164), 164              & 256         \\
100                          & ResNet110     & 62.47 \scriptsize{$\pm$ 4.2} & 9.43 \scriptsize{$\pm$ 1.6} & 60.94 \scriptsize{$\pm$ 4.1} & {\ul 12.52 \scriptsize{$\pm$ 1.9}} & 62.42 \scriptsize{$\pm$ 4.8} & \textbf{13.48 \scriptsize{$\pm$ 3.6}}                                                        & 14.48, (200, 164), 164              & 256         \\
                          & DenseNet40    & 60.0 \scriptsize{$\pm$ 0.2} & 1.82 \scriptsize{$\pm$ 0.1} & 60.0 \scriptsize{$\pm$ 0.5} & \textbf{2.11 \scriptsize{$\pm$ 0.1}} & 61.97 \scriptsize{$\pm$ 0.6} & {\ul 2.0 \scriptsize{$\pm$ 0.1}}
                          & 10.47, (100, 100), 300              & 256         \\ \hline
\end{tabular}
\vspace{-0.5em}
\label{table_1}
\end{table*}
\begin{table*}[t]
\centering
\caption{\blue{Adversarial trained model on CIFAR10 evaluated on Autoattack with $\epsilon=8 / 255$. (*) indicates extra unlabeled data used.
\\
\vspace{-0.35em}
Bold and underlined numbers denote the best and 2nd best Autoattack accuracy in each row.}}
\vspace{-1em}
\begin{tabular}{c|cc|cc|cc|cc}
\hline
Approach; Model          & \multicolumn{2}{c|}{Current SOTA} & \multicolumn{2}{c|}{\begin{tabular}[c]{@{}c@{}}PoE-motivated \\ $\eta^1=1/\Ls_{\text{est}}$ (Ours)\end{tabular}} & \multicolumn{2}{c|}{\begin{tabular}[c]{@{}c@{}}Largest convergent \\ $\eta^1=2/\Ls_{\text{est}}$ (Ours)\end{tabular}} & \multicolumn{2}{c}{Section \ref{subsec:Lips_esti}, \ref{subsec:batch} Parameters} \\ \hline
             & Clean      & Autoattack             & Clean                                       & Autoattack                                                 & Clean                               & Autoattack                                      & $\Ls_{\text{est}}$, (M, N), epochs        & Batch       \\ \hline
TRADES; WRN34-10 &
84.81 \scriptsize{$\pm$ .29} & 52.12 \scriptsize{$\pm$ .09} & 84.51 \scriptsize{$\pm$ .19} & \textbf{52.56 \scriptsize{$\pm$ .26}} & 83.44 \scriptsize{$\pm$ .15} & {\ul 52.27 \scriptsize{$\pm$ .04}} &
7.497, (99, 25), 75 & 128                        \\ 
PGD-AT; ResNet-50 &
85.98 \scriptsize{$\pm$ .09} & 42.66 \scriptsize{$\pm$ .08} & 86.21 \scriptsize{$\pm$ .32} & \textbf{43.03 \scriptsize{$\pm$ .13}} & 86.35 \scriptsize{$\pm$ .11} & {\ul 42.98 \scriptsize{$\pm$ .16}} &
10.40, (55, 150), 150 & 256                        \\ 
RST(*); WRN28-10 &      
89.48 \scriptsize{$\pm$ .05} &59.38 \scriptsize{$\pm$ .14} & 89.5 \scriptsize{$\pm$ .15} & {\ul 59.6 \scriptsize{$\pm$ .11}} & 89.48 \scriptsize{$\pm$ .07} & \textbf{59.7 \scriptsize{$\pm$ .08}} &
13.46, (160, 200), 200 & 256                        \\ 
\hline
\end{tabular}
\vspace{-1.5em}
\label{table_2}
\end{table*}

Following implementation details in Sections \ref{subsec:Lips_esti}, \ref{subsec:batch}, we perform standard training and AT with baseline, \PoE{}-motivated and largest convergent schedules. The clean \& adversarial accuracy (across 5 random seeds) \& relevant parameters for standard training is presented in Table \ref{table_1}; for AT in Table \ref{table_2}. For standard training, we analyze ResNet20, 50, 110 \cite{he2016resnet} \& DenseNet 40 \cite{huang2017densenet} on CIFAR10 \& the latter 3 on CIFAR100. The baseline training starts at $\eta^1=0.1$. We test each trained model on a 20-step PGD adversary with $\epsilon=\frac{1}{255}$ \& step-size $\frac{0.1}{255}$. \blue{We also use the PyTorch baseline's standard weight decay $1$e-$4$, momentum $0.9$, and a step schedule with learning rate scaled down by 10 at epochs $81, 122$ for ResNets \& $150, 225$ for DenseNet.} For CIFAR10 AT, we train ResNet50 in PGD-AT framework \cite{madry2017towards}; WideResNet (WRN) 34-10 \cite{zagoruyko2016wide} in TRADES \cite{zhang2019theoretically}; WRN 28-10 in RST \cite{carmon2019unlabeled}. PGD-AT \& TRADES decay $\eta$ by 10 every 50 epochs \& once at the 75th epoch respectively. RST uses 500K additional unlabelled images from the TinyImages dataset \cite{carmon2019unlabeled} \& a cosine annealing schedule. \blue{We follow the SOTA code of all three AT methods for other hyperparameters \&} test on Autoattack with $\epsilon=8/{255}$.

\section{DISCUSSION AND FUTURE WORK} \label{sec:conclusion}
Table \ref{table_1} shows a clear improvement in PGD attack accuracy over baseline (while maintaining similar clean accuracy) with both the largest convergent and PoE-motivated schedules. This demonstrates that our approach is promising in practical SGD. In Table \ref{table_2}, across various AT frameworks/models evaluated on Autoattack (where small improvements are considered noteworthy), we have consistent increase in Autoattack accuracy over SOTA and continued clean accuracy similarity with our schedules. Thus, we note that, even with the varied dynamics of AT frameworks, our approach acts as a 'force-multiplier' for robustness.


Lastly, based on the success of the largest convergent schedule \& our conservative upper bound, we conjecture that starting with a learning rate just below $\frac{2}{\Ls}$ also guarantees \PoE{} of GD. An immediate next step includes proving the conjecture. In addition, future work can focus on extending our sufficient condition proof (for PoE of GD) to PoE of SGD and its variants.

\normalem 
\bibliography{root.bib}
\bibliographystyle{IEEEtran}

\newpage
\section*{APPENDIX}

\begin{algorithm}[h]
\SetAlgoNoLine%
\SetKwInOut{Input}{Input}
\SetKwInOut{Output}{Output}
\Input{All saved gradients and parameters $(\nabla l(\theta^i), \theta^i)$, M, N, intial shape choices}
\Output{$\Ls_{\text{est}}$ }
\textbf{for} shape$_0$ in initial shape choices \textbf{do}\\
\quad\quad \textbf{for} j = 1, \ldots, M \textbf{do}\\
\quad\quad \quad\quad Sample $N$ points: $(\nabla l(\theta^i), \theta^i)_{i=1,\ldots,N}$.\\
\quad\quad \quad\quad Compute $N/2$ slopes between consecutive pairs of points:
$$
    s_i = \frac{||\nabla l(\theta^{i+1}) - \nabla l (\theta^i)||_2}{||\theta^{i+1} - \theta^i||_2}, \;\;\; i=1, 3, 5, ..., N.
$$\\
\quad\quad \quad\quad Compute maximum of the $N/2$ slopes: $l_j = \max \{s_1, s_3, s_5, ..., s_{N}\}$.\\
\quad\quad \textbf{end for}\\
\quad\quad (shape, location, scale) $\leftarrow$ Fit three parameter reverse Weibull distribution to $\{l_1, \ldots, l_M\}$ given initial shape value = shape$_0$. \\
\quad\quad p-value $\leftarrow$ Kolmogrov-Smirnov goodness-of-fit test. \\
\textbf{end for}\\
\Return{scale corresponding to largest p-value}
\caption{Estimation of Certified Lipschitz Constant}
\label{alg_est}
\end{algorithm}


\section{Estimation of Certified Lipschitz Constant \texorpdfstring{$\Ls$}{L} via Extreme Value Theory}
\label{app:full_lips}

We detail the estimation algorithm previously  discussed in Section \ref{subsec:Lips_esti} in Algorithm \ref{alg_est}. We depict our heuristic for (M, N) tuning in Figure \ref{fig:opt1} for a ResNet-20 model \cite{he2016resnet} and a significance value of $\alpha=0.55$. The row corresponding to $(M, N)=(200,100)$ in Figure \ref{fig:opt1}(a) is the only one that has both, p-values larger \& lesser than $\alpha$ (satisfying our heuristic) and the largest of these corresponds to a Lipschitz constant of 3.5258 (see red box on Figure \ref{fig:opt1}(b)). Thus, our estimated Lipschitz constant is $\Ls_{\text{est}}=3.5258$. Please find the complete extended version of Figure \ref{fig:opt1}(a), \textit{i.e.} a complete heat map of all (M,N) tuples vs p-values in Figure \ref{fig:opt1_complete}.

\begin{figure*}[t]
    \centering
    \includegraphics[width=0.8\linewidth]{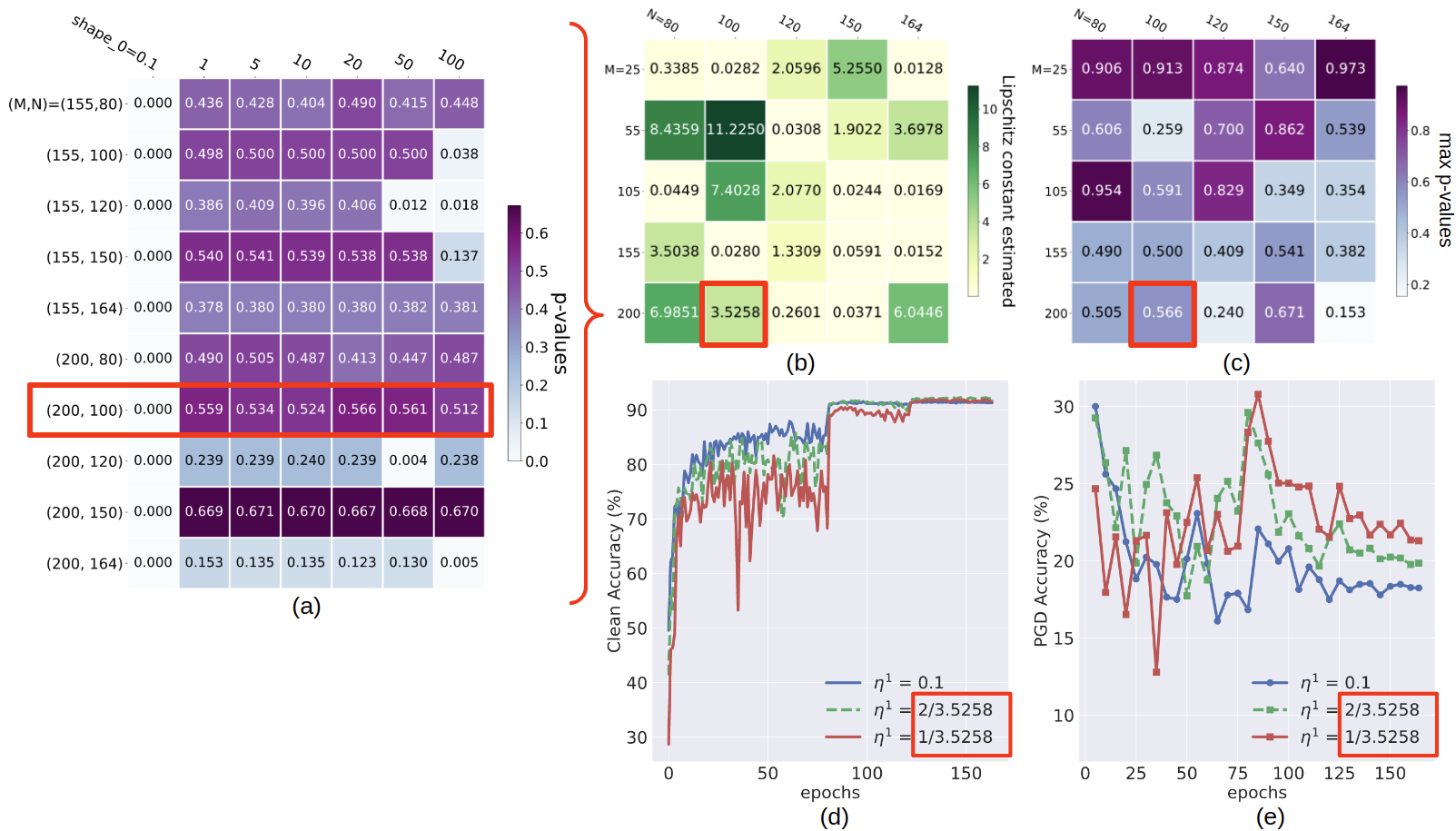}
    \caption{Estimating $\Ls$ for ResNet-20 model in standard classifier training using initial shape choices $=\{0.1,1,5,10,20,50,100\}$, $M$ $\in \{25,55,105,155,200\}$, $N$ $\in\{80,100,120,150,164\}$, and significance value $\alpha=0.55$. Here, (a) shows a heat map of p-values for some (M, N) tuples vs initial shape (shape$_0$) values; (b) shows Lipschitz constant estimates for all M, N values in a heat map; (c) depicts their corresponding max p-values (also in a heat map); (d) and (e) are reproductions of Figure \ref{opt1_only_plot} and depict the variation of clean and PGD attack accuracy as a function of epochs for all three schedules (baseline, PoE-motivated and largest convergent).}
    \label{fig:opt1}
\end{figure*}

\textbf{Analysis of time complexity and memory overhead.} Algorithm \ref{alg_est} maintains a $O(mn)$ time complexity which is negligible in comparison to the model training time ($m, n$ are the number of values of hyperparamters $M, N$ tried in Algorithm \ref{alg_est}). Our primary increase in training time is a consequence of having to train a baseline model and then another model with a \PE{} schedule. This results in a $2\times$ increase in time complexity. We hope future work can help boost performance (for example, by adapting learning rates online to satisfy \PoE{} conditions). We also note that there is a small memory overhead in having to save gradients plus parameters after every epoch for use in Lipschitz estimation post training. This overhead is given by $O(n_{\text{epochs}} (n_{\text{params}} + n_{\text{pixels}}))$ where $n_{\text{epochs}}$, $n_{\text{params}}$, and $n_{\text{pixels}}$ denote the number of epochs, model parameters and input image pixels respectively.

\textbf{Limitations of the estimation algorithm.} 
The estimation algorithm is inherently random because it depends on the gradients and parameters saved during the training process which can change with each run even when using the same random seed. Yet, the advantage of our results are that future work can introduce a better estimation algorithm (preferably with less inherent randomness) for $\Ls$ and use it in conjunction with our PoE-motivated or largest convergent learning rate schedule for increased adversarial robustness.

\section{Details of Adversarial Training Frameworks and Autoattack} \label{app:AT_frameworks}
We describe the adversarial training frameworks analyzed in this work and the autoattack benchmark used to evaluate models trained in said frameworks below.

\textbf{PGD-AT} \cite{madry2017towards}: The general adversarial training min-max optimization problem is given by
\begin{align*}
    \arg \min_{\theta} \E_{(X, Y) \in \Xs \times \Ys} \left[ \max_{\delta \in \mathbb{S}} L (h_{\Theta}(X+\delta), Y) \right]
\end{align*}
where $\mathbb{S}_p = \{\delta \; | \; ||\delta||_p < \epsilon\}$, $X, Y$ represent batch training data \& labels and the rest of the notation is defined in Section \ref{sec:prob}. We are primarily concerned with $l_{\infty}$ perturbations in this work which is why we have $\mathbb{S} = \mathbb{S}_{\infty}$. The inner maximization is solved by projected gradient descent (PGD) on the negative loss function (for $K$ steps with $\alpha$ step size) to get an adversarial example represented as $X^{(K)} = X + \delta^{(K)}$. The perturbed data point in the $(t+1)$-th step (\textit{i.e.} $X^{(t+1)}$) is given by\\
\begin{align*}
    X^{(t+1)} = \prod_{X+\mathbb{S}} (X^{(t)} +  \alpha \; \text{sgn}(\nabla_X L(h_{\Theta}(X^{(t)}), Y)))
\end{align*}
with initialization $X^{(0)} = X + \delta^{(0)}$ where $\delta^{(0)}$ can be set to $0$ or to any random point within $\mathbb{S}$. The latter case is called PGD with random initialization. The $\prod_{x+\mathbb{S}}$ denotes projecting perturbations of perturbed data points into the set $\mathbb{S}$.

\textbf{TRADES} \cite{zhang2019theoretically}: In TRADES, a theoretically motivated surrogate loss that balances the trade-off between standard and robust accuracy is minimized. The TRADES loss function is given by,
\begin{align*}
    L^{\text{TRADES}}_{\Theta} = L(h_{\Theta}(X), Y) + \beta \max_{\delta \in \mathbb{S}} D_{\text{KL}} (h_{\Theta}(X+\delta) || h_{\Theta}(X))
\end{align*}
where $D_{\text{KL}}$ represents Kullback–Leibler (KL) divergence and $\beta$ is a hyperparameter that controls the aforementioned trade-off.

\textbf{RST} \cite{carmon2019unlabeled}: In RST, a separate standard model is trained over CIFAR10 and used to generate pseudo-labels for unlabelled images from the TinyImages dataset \cite{carmon2019unlabeled}. Then a robust model is trained over the unlabelled data and its pseudo-labels by minimizing the TRADES loss given above. By this self-supervised training process, an adversarial-trained robust classifier is obtained.



\textbf{Autoattack} \cite{croce2020reliable_autoattack}: Autoattack consists of 4 attacks -- Auto-PGD on cross entropy loss (white-box), Auto-PGD on difference of logits ratio loss (also white-box), Fast adaptive boundary attack (black-box) and Square attack (also black-box). Evaluation on autoattack has very little (0.01\%) to no variance in different runs. Moreover, it has only one hyperparameter $\epsilon$ (usually set to $8/255$) while all others are fixed and abstracted away from the evaluation making comparison across models and frameworks easy.

\textbf{Hyperparameters for Adversarial Training}: We set momentum to $0.9$ in all three frameworks and set weight decay to 5e-4 in PGD-AT \& RST; 2e-4 in TRADES. The $\beta$ parameter in the TRADES formulation of adversarial loss (which is also used in RST) is set to $0.6$. It does not exist for the adversarial loss in PGD-AT. The same perturbation budget of $\epsilon=8/255$, attack steps $=10$, and attack step-size of $0.007$ are used in all three methods. These hyperparameters are obtained from the current SOTA of the three frameworks as given in \cite{madry2017towards, zhang2019theoretically, carmon2019unlabeled}.

\section{Additional Details of Experiments}

\textbf{Data Augmentation for Standard and Adversarial Training}: Following the common practice for CIFAR datasets (and following the SOTA implementations of all 4 adversarial training frameworks), training images are augmented with random crops (padding by 4 pixels and cropping to 32 $\times$ 32) and random horizontal flips.

\textbf{Computation resources used in running experiments}: We ran the experiments on either two Nvidia GeForce RTX 3090 GPUs (each with 24 GB of memory) or two Nvidia Quadro RTX 6000 GPUs (each with 24 GB of memory). The CPUs used were Intel Xeon Gold processors @ 3 GHz.

\begin{figure}[H]
\centering
  \includegraphics[width=\linewidth]{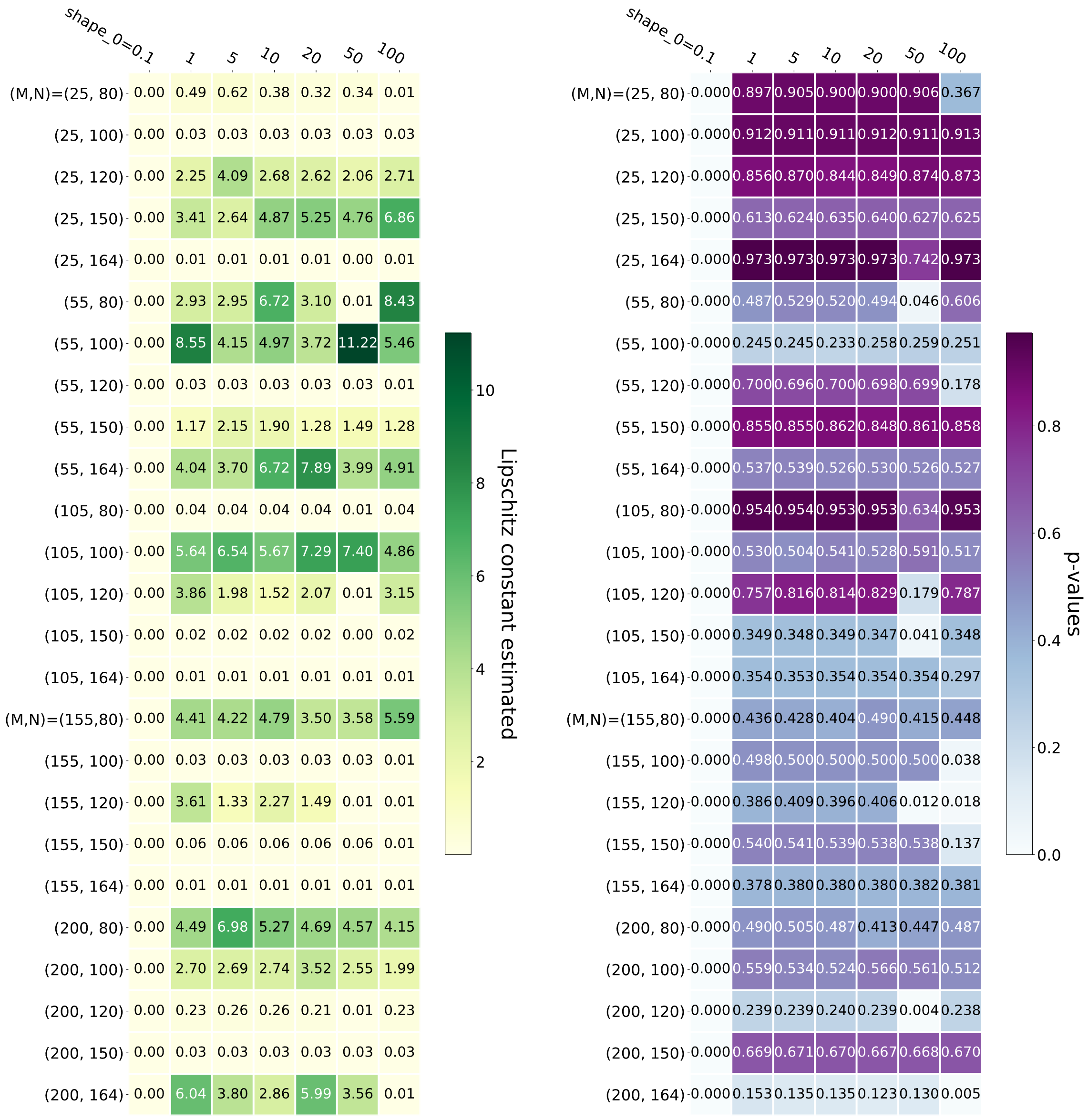}
  \caption{[LEFT] heat map of Lipschitz constants estimated (\textit{i.e.} fitted scale parameter) for various (M, N) tuples vs various initial shape parameters. [RIGHT] heat map of corresponding p-values for various (M, N) tuples vs various initial shape parameters. }
  \label{fig:opt1_complete}
\end{figure}

\textbf{Code bases utilized}: The code for LeNet5 on MNIST is based on \url{https://github.com/ChawDoe/LeNet5-MNIST-PyTorch} (No license). Standard training in CIFAR10, CIFAR100 for all models is based on code from a repository of PyTorch baselines at \url{https://github.com/bearpaw/pytorch-classification} (MIT license). We used the Advertorch python library at \url{https://github.com/BorealisAI/advertorch} (GNU general public license) for PGD implementation in standard training.

In adversarial training, the code for PGD-AT framework \cite{madry2017towards} is from \url{https://github.com/MadryLab/robustness} (MIT license), the code for TRADES \cite{zhang2019theoretically} framework is from \url{https://github.com/yaodongyu/TRADES} (MIT license) and the code for RST \cite{carmon2019unlabeled} is from \url{https://github.com/yaircarmon/semisup-adv} (MIT license). Our Lipschitz constant estimation code is based on previous work by \cite{weng2018evaluating} and can be found at \url{https://github.com/huanzhang12/CLEVER} (Apache license).

\end{document}